\documentclass{article}
\usepackage[inline]{enumitem}

\usepackage{microtype}
\usepackage{graphicx}
\usepackage{epstopdf}
\usepackage{subfigure}
\usepackage{booktabs} % for professional tables
\usepackage{extarrows}

\usepackage{amsmath}   
\usepackage{amssymb}
\usepackage{wasysym}
\usepackage{multirow}
\usepackage[dvipsnames]{xcolor}
\usepackage{amsthm}
\usepackage{bbm}
\usepackage{tikz}
\usepackage{enumerate}
\usetikzlibrary{calc}
\usetikzlibrary{arrows.meta}
\usepackage{footmisc}
\newtheorem{theorem}{Theorem}
\newtheorem{lemma}{Lemma}

\usepackage{hyperref}

\newcommand{\amauri}[1]{}
\newcommand{\felix}[1]{}
\newcommand{\kilian}[1]{}
\newcommand{\chris}[1]{}
\newcommand{\tao}[1]{}
\newcommand{\tianyi}[1]{}

\newcommand{\Method}{Simple Graph Convolution}
\newcommand{\method}{SGC}

\definecolor{color_sn}{HTML}{79a6f6}
\definecolor{color_cheby}{HTML}{a8c0f3}
\definecolor{color_gcn}{HTML}{c8daf4}
\definecolor{color_sgr}{HTML}{e1ebf7}
\definecolor{color_perceptron}{HTML}{ffd6d6}
\definecolor{color_mlp}{HTML}{ffbfbf}
\definecolor{color_filter}{HTML}{fff9db}
\definecolor{color_cnn}{HTML}{fff3ba}

\definecolor{myred}{HTML}{e53935}
\definecolor{myblue}{HTML}{0277bd}

\definecolor{modelblue}{HTML}{1034A6}

\usepackage{amsmath,amsfonts,bm}

\def\eqref#1{equation~\ref{#1}}
\def\1{\bm{1}}

\def\rvg{{\mathbf{g}}}

\def\rvv{{\mathbf{v}}}

\def\rvx{{\mathbf{x}}}
\def\rvy{{\mathbf{y}}}

\def\rmA{{\mathbf{A}}}

\def\rmD{{\mathbf{D}}}

\def\rmG{{\mathbf{G}}}
\def\rmH{{\mathbf{H}}}
\def\rmI{{\mathbf{I}}}

\def\rmS{{\mathbf{S}}}

\def\rmU{{\mathbf{U}}}

\def\rmX{{\mathbf{X}}}

\DeclareMathAlphabet{\mathsfit}{\encodingdefault}{\sfdefault}{m}{sl}
\SetMathAlphabet{\mathsfit}{bold}{\encodingdefault}{\sfdefault}{bx}{n}

\newcommand{\softmax}{\mathrm{softmax}}
\newcommand{\relu}{\mathrm{ReLU}}

\DeclareMathOperator{\Tr}{Tr}

\graphicspath{{figures/}}

\usepackage[accepted]{icml2019}

\icmltitlerunning{Simplifying Graph Convolutional Networks}

\begin{document}

\twocolumn[
% \icmltitle{Keep graph convolutions as simple as possible}
\icmltitle{Simplifying Graph Convolutional Networks}



\icmlsetsymbol{equal}{*}

\begin{icmlauthorlist}
\icmlauthor{Felix Wu}{equal,cornell}
\icmlauthor{Tianyi Zhang}{equal,cornell}
\icmlauthor{Amauri Holanda de Souza Jr.}{equal,cornell,ifce}
\icmlauthor{Christopher Fifty}{cornell}
\icmlauthor{Tao Yu}{cornell}
\icmlauthor{Kilian Q. Weinberger}{cornell}
\end{icmlauthorlist}

\icmlaffiliation{cornell}{Cornell University}
\icmlaffiliation{ifce}{Federal Institute of Ceara (Brazil)}

\icmlcorrespondingauthor{Felix Wu}{fw245@cornell.edu}
\icmlcorrespondingauthor{Tianyi Zhang}{tz58@cornell.edu}
% \icmlcorrespondingauthor{Eee Pppp}{ep@eden.co.uk}

\icmlkeywords{Machine Learning, ICML}
\vskip 0.3in
]



\printAffiliationsAndNotice{\icmlEqualContribution} % otherwise use the standard text.

\begin{abstract}
Graph Convolutional Networks (GCNs) and their variants have experienced significant attention and have become the de facto methods for learning graph representations. 
GCNs derive inspiration primarily from recent deep learning approaches, and as a result, may inherit unnecessary complexity and redundant computation. 
In this paper, we reduce this excess complexity through successively removing nonlinearities and collapsing weight matrices between consecutive layers. 
We theoretically analyze the resulting linear model and show that it corresponds to a fixed low-pass filter followed by a linear classifier. 
Notably, our experimental evaluation demonstrates that these simplifications do not negatively impact accuracy in many downstream applications.
Moreover, the resulting model scales to larger datasets, is naturally interpretable, and yields up to two orders of magnitude speedup over FastGCN.

\end{abstract}

\section{Introduction}
Graph Convolutional Networks (GCNs) \cite{gcn} are an efficient variant of Convolutional Neural Networks (CNNs) on graphs. 
GCNs stack layers of learned first-order spectral filters followed by a nonlinear activation function to learn graph representations.
Recently, GCNs and subsequent variants have achieved state-of-the-art results in various application areas, including but not limited to citation networks \cite{gcn}, social networks \cite{FastGCN}, applied chemistry \cite{liao2018lanczosnet}, natural language processing \cite{textGCN, han2012geolocation, relation-extraction}, and computer vision \cite{wang2018zero, ADGPM}.  

Historically, the development of machine learning algorithms has followed a clear trend from initial simplicity to need-driven complexity. For instance, limitations of the linear Perceptron \cite{rosenblatt1958perceptron} motivated the development of the more complex but also more expressive neural network (or multi-layer Perceptrons, MLPs)~\cite{rosenblatt1961principles}. Similarly, simple pre-defined linear image filters~\cite{sobel19683x3,harris1988combined} eventually gave rise to nonlinear CNNs with learned convolutional kernels ~\cite{waibel1989phoneme,lecun1989backpropagation}. 
As additional algorithmic complexity tends to complicate theoretical analysis and obfuscates understanding, it is typically only introduced  for applications where simpler methods are insufficient.  Arguably, most classifiers in real world applications are still linear (typically logistic regression), which are straight-forward to optimize and easy to interpret.

\begin{figure*}[h!]
    \centering
    \includegraphics[width=1.0\linewidth]{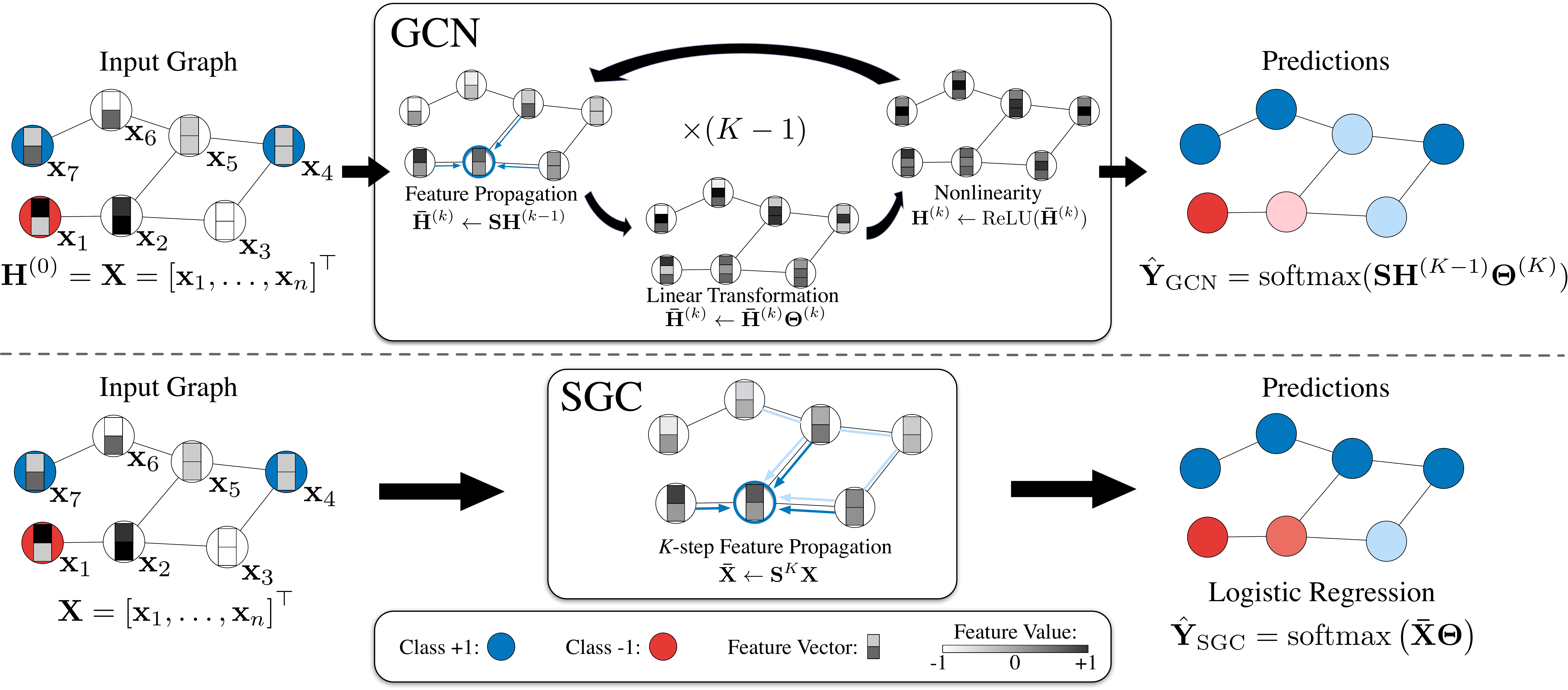}
    \caption{Schematic layout of a GCN v.s. a SGC. \textit{Top row:} The GCN  transforms the feature vectors repeatedly throughout $K$ layers and then applies a linear classifier on the final representation. \textit{Bottom row:}  the \method{} reduces the entire procedure to a simple feature propagation step followed by standard logistic regression. }
    \label{fig:method}
\end{figure*}

However, possibly because GCNs were proposed after the recent ``renaissance" of neural networks, they tend to be a rare exception to this trend. GCNs are built upon multi-layer neural networks, and were never an extension of a simpler (insufficient) linear counterpart. 

In this paper, we observe that GCNs inherit considerable  complexity from their deep learning lineage, which can be burdensome and unnecessary for less demanding applications. Motivated by the glaring historic omission of a simpler predecessor, we aim to derive the simplest linear model that ``could have'' preceded the GCN, had a more ``traditional'' path been taken. We reduce the excess complexity of GCNs by repeatedly removing the nonlinearities between GCN layers and collapsing the resulting function into a single linear transformation. We empirically show that the final linear model exhibits comparable or even superior performance to GCNs on a variety of tasks while being computationally more efficient and fitting significantly fewer parameters. We refer to this simplified linear model as \Method{} (\method{}). 

In contrast to its nonlinear counterparts, the  \method{} is intuitively interpretable and 
we provide a theoretical analysis from the graph convolution perspective. 
Notably, feature extraction in \method{} corresponds to a single fixed filter applied to each feature dimension. 
\citet{gcn} empirically observe that the ``renormalization trick", i.e. adding self-loops to the graph, improves accuracy, and we demonstrate that this method effectively shrinks the graph spectral domain, resulting in a low-pass-type filter when applied to \method{}. 
Crucially, this filtering operation gives rise to locally smooth features across the graph~\cite{Bruna13}.

Through an empirical assessment on node classification benchmark datasets for citation and social networks, we show that the \method{} achieves comparable performance to GCN and other state-of-the-art graph neural networks. However, it is significantly faster, and even outperforms  FastGCN~\citep{FastGCN} by  up to two orders of magnitude on the largest dataset (Reddit) in our evaluation. 
Finally, we demonstrate that \method{} extrapolates its effectiveness to a wide-range of downstream tasks. In particular, \method{} rivals, if not surpasses, GCN-based approaches on text classification, user geolocation, relation extraction, and zero-shot image classification tasks. 
The code is available on Github\footnote{\url{https://github.com/Tiiiger/SGC}}.

\section{\Method{}}
%!TEX root=main.tex
We follow \citet{gcn} to introduce GCNs (and subsequently \method{}) in the context of node classification. Here, GCNs take a graph with some labeled nodes as input and generate label predictions for all graph nodes. Let us formally define such a graph as ${\mathcal{G}} = ({\mathcal{V}}, \rmA)$, where $\mathcal{V}$ represents the vertex set consisting of nodes $\{v_1, \dots, v_n\}$, and 
$\rmA\in\mathbb{R}^{n \times n}$ is a symmetric (typically sparse) adjacency matrix
where $a_{ij}$ denotes the edge weight between nodes $v_i$ and $v_j$. A missing edge is represented through $a_{ij} = 0$.
 We define the degree matrix $\rmD = \text{diag}(d_1, \dots, d_n)$ as a diagonal matrix where each entry on the diagonal is equal to the row-sum of the adjacency matrix $d_i =  \sum_j a_{ij}$.

% machine learning on graphs background
Each node $v_i$ in the graph has a corresponding $d$-dimensional feature vector $\rvx_i \in \mathbb{R}^{d}$. The entire feature matrix $\rmX \in \mathbb{R}^{n \times d}$ stacks $n$ feature vectors on top of one another, $\rmX=\left[\rvx_1,\dots,\rvx_n\right]^\top$. 
Each node belongs to one out of $C$ classes and can be labeled with a $C$-dimensional one-hot vector $\rvy_i\in\{0,1\}^C$.
We only know the labels of a subset of the nodes and want to predict the unknown labels.

\subsection{Graph Convolutional Networks}
Similar to CNNs or MLPs, GCNs  learn a new feature representation for the feature $\mathbf{x}_i$ of each node over multiple layers, which is subsequently used as input into a linear classifier.  For the $k$-th graph convolution layer, we denote the input node representations of all nodes by the matrix  $\mathbf{H}^{(k-1)}$ and the output node representations $\mathbf{H}^{(k)}$. Naturally, the initial node representations are just the original input features: 
\begin{equation} \label{eq:initial_feature}
    \mathbf{H}^{(0)} = \mathbf{X}, 
\end{equation}
which serve as input to the first GCN layer. 

A $K$-layer GCN is identical to applying a $K$-layer MLP to the feature vector $\mathbf{x}_i$ of each node in the graph, except that the hidden representation of each node is averaged with its neighbors at the beginning of each layer. In each graph convolution layer, node representations are updated in three stages: feature propagation, linear transformation, and a pointwise nonlinear activation (see \autoref{fig:method}). For the sake of clarity, we describe each step in detail.

% feature propagation
\paragraph{Feature propagation} is what distinguishes a GCN from an MLP. 
At the beginning of each layer the features $\mathbf{h}_i$ of each node $v_i$ are averaged with  the feature vectors  in its local neighborhood, 
\begin{equation}
    \bar{\mathbf{h}}_i^{(k)} \leftarrow \frac{1}{d_i + 1} \mathbf{h}_i^{(k-1)}+\sum_{j=1}^n \frac{a_{ij}}{\sqrt{(d_i + 1) (d_j + 1)}}\mathbf{h}_j^{(k-1)}.\label{eq:update}
\end{equation}
More compactly, we can express this update over the entire graph as a simple matrix operation.  Let $\rmS$ denote the ``normalized'' adjacency matrix with added self-loops, 
\begin{align} 
\label{eq:propagation_matrix}
    \mathbf{S} & = \tilde{\rmD}^{-\frac{1}{2}} \tilde{\rmA} \tilde{\rmD}^{-\frac{1}{2}},
\end{align}
where $\tilde{\rmA} = \rmA + \rmI$ and $\tilde{\rmD}$ is the degree matrix of $\tilde{\rmA}$. The simultaneous update in~\autoref{eq:update} for all nodes becomes a simple sparse matrix multiplication
\begin{equation}
    \bar{\mathbf{H}}^{(k)} \leftarrow \mathbf{S} \mathbf{H}^{(k-1)}.
\end{equation}
Intuitively, this step smoothes the hidden representations locally along the edges of the graph and ultimately encourages similar predictions among locally connected nodes.

% Linear Transformation
\paragraph{Feature transformation and nonlinear transition.} 
After the local smoothing, a GCN layer is identical to a standard MLP.  Each layer is associated with a learned weight matrix $\boldsymbol{\Theta}^{(k)}$, and the smoothed hidden feature representations are transformed linearly. 
Finally, a nonlinear activation function such as $\relu$ is applied pointwise before outputting feature representation $\rmH^{(k)}$. In summary, the representation updating rule of the $k$-th layer is: 
\begin{align} \label{eq:gcn_propagation}
    \mathbf{H}^{(k)} & \leftarrow  \relu \left( \bar{\mathbf{H}}^{(k)} \boldsymbol{\Theta}^{(k)}\right). 
\end{align}
The pointwise nonlinear transformation of the $k$-th layer is followed by the feature propagation of the $(k+1)$-th layer.
% node classification
\paragraph{Classifier.} For node classification, and similar to a standard MLP, the last layer of a GCN predicts the labels using a \textit{softmax} classifier. Denote the class predictions for $n$ nodes as $\hat{\mathbf{Y}} \in \mathbb{R}^{n\times C}$ where 
$\hat{y}_{ic}$ denotes the probability of node $i$ belongs to class $c$.
The class prediction $\hat{\mathbf{Y}}$ of a $K$-layer GCN can be written as:
\begin{align}
\hat{\mathbf{Y}}_{\text{GCN}} & = \softmax\left( \rmS \mathbf{H}^{(K-1)} \boldsymbol{\Theta}^{(K)}\right),
\end{align}
where $\softmax(\rvx) = \exp(\rvx) / \sum_{c=1}^C \exp(x_c)$ acts as a normalizer across all classes. 

\subsection{Simple Graph Convolution}
% transition to SGC
In a traditional MLP, deeper layers increase the expressivity because it allows the creation of feature hierarchies, e.g. features in the second layer build on top of the features of the first layer. In GCNs, the layers have a second important function: in each layer the hidden representations are averaged among neighbors that are one hop away. This implies that after $k$ layers a node obtains feature information from all nodes that are $k-$hops away in the graph. This effect is similar to convolutional neural networks, where depth increases the receptive field of internal features~\cite{hariharan2015hypercolumns}.  Although convolutional networks can benefit substantially from increased depth~\cite{huang2016deep}, typically MLPs obtain little benefit beyond 3 or 4 layers. 

\paragraph{Linearization.}
We hypothesize that the nonlinearity between GCN layers is not critical - but that the majority of the benefit arises from the local averaging. We therefore remove the nonlinear transition functions between each layer and only keep the final softmax (in order to obtain probabilistic outputs). 
The resulting model is linear, but still has the same increased ``receptive field'' of a $K$-layer GCN,
\begin{align}
    \hat{\mathbf{Y}} & = \softmax\left(\mathbf{S}\ldots\mathbf{S}\mathbf{S} \mathbf{X} \boldsymbol{\Theta}^{(1)}\boldsymbol{\Theta}^{(2)}\ldots\boldsymbol{\Theta}^{(K)} \right).
    \end{align}
To simplify notation we can collapse the repeated multiplication with the normalized adjacency matrix $\rmS$ into a single matrix by raising $\rmS$ to the $K$-th power, $\rmS^K$. Further, we can reparameterize our weights into a single matrix  $\boldsymbol{\Theta}=\boldsymbol{\Theta}^{(1)} \boldsymbol{\Theta}^{(2)} \ldots \boldsymbol{\Theta}^{(K)}$.  The resulting classifier becomes
\begin{equation}
    \hat{\mathbf{Y}}_{\text{\method{}}}=\softmax\left(\mathbf{S}^K \mathbf{X} \boldsymbol{\Theta} \right),\label{eq:SGC}
\end{equation}
which we refer to as \Method{} (\method{}). 

% SGC conclusion 
\paragraph{Logistic regression.} \autoref{eq:SGC} gives rise to a  natural and intuitive interpretation of \method{}: by distinguishing between feature extraction and classifier, \method{} consists of a fixed (i.e., parameter-free) feature extraction/smoothing component $\bar{\rmX}= \rmS^K \rmX$ followed by a linear logistic regression classifier $\hat{\mathbf{Y}}=\softmax(\bar{\rmX}\boldsymbol{\Theta})$. Since the computation of $\bar{\rmX}$ requires no weight it is essentially equivalent to a feature pre-processing step and the entire training of the model reduces to straight-forward multi-class logistic regression on the pre-processed features $\bar \rmX$.

\paragraph{Optimization details.} The training of logistic regression is a well studied convex optimization problem and can be performed with any efficient second order method or stochastic gradient descent~\cite{bottou2010large}. Provided the graph connectivity pattern  is sufficiently sparse, SGD naturally scales to very large graph sizes and the training of \method{} is  drastically faster than that of GCNs.

\section{Spectral Analysis}
\label{sec:analysis}
We now study \method{} from a graph convolution perspective. We demonstrate that \method{} corresponds to a fixed filter on the graph spectral domain. 
In addition, we show that adding self-loops to the original graph, i.e. the renormalization trick \citep{gcn}, effectively shrinks the underlying graph spectrum.
On this scaled domain, \method{} acts as a low-pass filter that produces smooth features over the graph. As a result, nearby nodes tend to share similar representations and consequently predictions.

\subsection{Preliminaries on Graph Convolutions}
Analogous to the Euclidean domain, graph Fourier analysis relies on the spectral decomposition of graph Laplacians. The graph Laplacian $\boldsymbol{\Delta} = \rmD - \rmA$ (as well as its normalized version $\boldsymbol{\Delta}_{\text{sym}} = \rmD^{-1/2}\boldsymbol{\Delta} \rmD^{-1/2}$) is a symmetric positive semidefinite matrix with eigendecomposition $\boldsymbol{\Delta} = \rmU \boldsymbol{\Lambda} \rmU^{\top}$, where $\rmU \in \mathbb{R}^{n \times n}$ comprises orthonormal eigenvectors and $\boldsymbol{\Lambda} = \text{diag}(\lambda_1, \dots, \lambda_n)$ is a diagonal matrix of eigenvalues. The eigendecomposition of the Laplacian allows us to define the Fourier transform equivalent on the graph domain, where eigenvectors denote Fourier modes and eigenvalues denote frequencies of the graph. In this regard, let $\rvx \in \mathbb{R}^n$ be a signal defined on the vertices of the graph. We define the graph Fourier transform of $\rvx$ as $\hat{\rvx} = \rmU^{\top}\rvx$, with inverse operation given by $\rvx = \rmU \hat{\rvx}$.
Thus, the graph convolution operation between signal $\rvx$ and filter $\rvg$ is
\begin{equation} \label{eq:spectral_conv_definition}
    \rvg * \rvx = \rmU \left( (\rmU^\top \rvg) \odot (\rmU^\top \rvx) \right) = \rmU \hat{\rmG} \rmU^\top \rvx,
\end{equation}
where $\hat{\rmG} = \text{diag}\left(\hat{g}_1, \dots, \hat{g}_n\right)$ denotes a diagonal matrix in which the diagonal corresponds to spectral filter coefficients. 

Graph convolutions can be approximated by $k$-th order polynomials of Laplacians
\begin{equation}
\label{eq:polynomial approximation}
    \rmU\hat{\rmG}\rmU^\top\rvx \approx \sum_{i=0}^{k} \theta_i \boldsymbol{\Delta}^i \rvx = \rmU \left( \sum_{i=0}^{k} \theta_i \boldsymbol{\Lambda}^i \right) \rmU^\top \rvx,
\end{equation}
where $\theta_i$ denotes coefficients. In this case, filter coefficients correspond to polynomials of the Laplacian eigenvalues, i.e., $\hat{\rmG} = \sum_i \theta_i \boldsymbol{\Lambda}^i$ or equivalently $\hat{g}(\lambda_j) = \sum_i \theta_i \lambda_j^i$.

Graph Convolutional Networks (GCNs) \cite{gcn} employ an affine approximation ($k=1$) of \autoref{eq:polynomial approximation} with coefficients $\theta_0=2\theta$ and $\theta_1=-\theta$ from which we attain the basic GCN convolution operation
\begin{equation} \label{eq:first-order-cheby}
    \rvg * \rvx = \theta (\rmI + \rmD^{-1/2} \rmA \rmD^{-1/2}) \rvx.
\end{equation}

In their final design, \citet{gcn} replace the matrix $\rmI + \rmD^{-1/2} \rmA \rmD^{-1/2}$ by a normalized version $\tilde{\rmD}^{-1/2} \tilde{\rmA} \tilde{\rmD}^{-1/2}$ where $\tilde{\rmA} = \rmA + \rmI$ and consequently $\tilde{\rmD} = \rmD + \rmI$, dubbed the \textit{renormalization trick}. Finally, by generalizing the convolution to work with multiple filters in a $d$-channel input and layering the model with nonlinear activation functions between each layer, we have the GCN propagation rule as defined in \autoref{eq:gcn_propagation}.

\begin{figure*}[tb] 
\centering
\includegraphics[width=.8\linewidth]{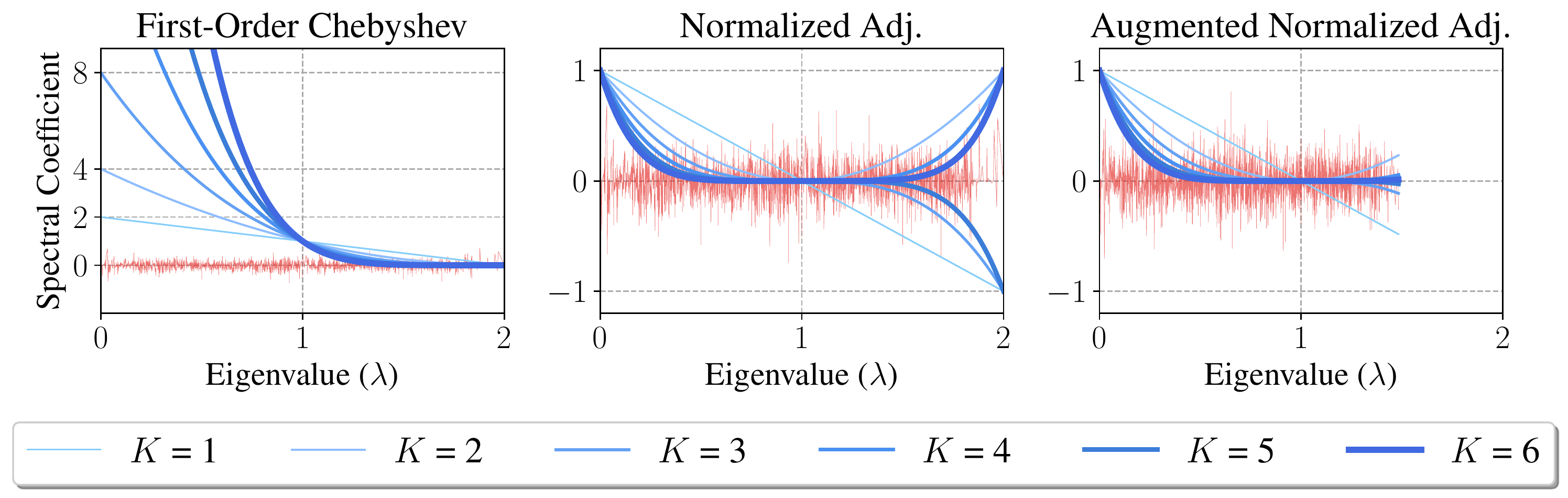}
\caption{Feature ({\color{myred}red}) and filters ({\color{myblue}blue}) spectral coefficients for different propagation matrices on Cora dataset ($3$rd feature).}
\label{fig:s_filters}
\end{figure*}

\subsection{\method{} and Low-Pass Filtering}
The initial first-order Chebyshev filter derived in GCNs corresponds to the propagation matrix $\rmS_{\text{1-order}} = \rmI + \rmD^{-1/2} \rmA \rmD^{-1/2}$ (see \autoref{eq:first-order-cheby}). Since the normalized Laplacian is $\boldsymbol{\Delta}_{\text{sym}} = \rmI - \rmD^{-1/2} \rmA \rmD^{-1/2}$, then $\rmS_{\text{1-order}} = 2\rmI - \boldsymbol{\Delta}_{\text{sym}}$. Therefore, feature propagation with $\rmS_{\text{1-order}}^K$ implies filter coefficients $\hat{g}_i = \hat{g}(\lambda_i) = (2 - \lambda_i)^K$, where $\lambda_i$ denotes the eigenvalues of $\boldsymbol{\Delta}_{\text{sym}}$. \autoref{fig:s_filters} illustrates the filtering operation related to $\rmS_{\text{1-order}}$ for a varying number of propagation steps $K \in \{1,\dots, 6\}$. As one may observe, high powers of $\rmS_{\text{1-order}}$ lead to exploding filter coefficients and undesirably over-amplify signals at frequencies $\lambda_i < 1$.

To tackle potential numerical issues associated with the first-order Chebyshev filter, \citet{gcn} propose the \textit{renormalization trick}. Basically, it consists of replacing $\rmS_{\text{1-order}}$ by the normalized adjacency matrix after adding self-loops for all nodes. We call the resulting propagation matrix the augmented normalized adjacency matrix $\tilde{\rmS}_{\text{adj}} = \tilde{\rmD}^{-1/2}\tilde{\rmA}\tilde{\rmD}^{-1/2}$, where $\tilde{\rmA}=\rmA+\rmI$ and $\tilde{\rmD}=\rmD+\rmI$. Correspondingly, we define the augmented normalized Laplacian $\tilde{\boldsymbol{\Delta}}_{\text{sym}} = \rmI -  \tilde{\rmD}^{-1/2}\tilde{\rmA}\tilde{\rmD}^{-1/2}$. Thus, we can describe the spectral filters associated with $\tilde{\rmS}_{\text{adj}}$ as a polynomial of the eigenvalues of the underlying Laplacian, i.e.,  $\hat{g}(\tilde{\lambda}_i) = (1 - \tilde{\lambda}_i)^K$, where $\tilde{\lambda}_i$ are eigenvalues of $\tilde{\boldsymbol{\Delta}}_{\text{sym}}$.

We now analyze the spectrum of $\tilde{\boldsymbol{\Delta}}_{\text{sym}}$ and show that adding self-loops to graphs shrinks the spectrum (eigenvalues) of the corresponding normalized Laplacian.

\begin{theorem}\label{thm:Lapla_eig_shrink} Let $\rmA$ be the adjacency matrix of an undirected, weighted, simple graph $\mathcal{G}$ without isolated nodes and with corresponding degree matrix $\rmD$. Let $\tilde{\rmA} = \rmA + \gamma \rmI$, such that $\gamma > 0$, be the augmented adjacency matrix with corresponding degree matrix $\tilde{\rmD}$. Also, let $\lambda_1$ and $\lambda_n$ denote the smallest and largest eigenvalues of $\boldsymbol{\Delta}_{\text{sym}}=\rmI -  \rmD^{-1/2}\rmA\rmD^{-1/2}$; similarly, let $\tilde{\lambda}_1$ and $\tilde{\lambda}_n$ be the smallest and largest eigenvalues of $\tilde{\boldsymbol{\Delta}}_{\text{sym}}=\rmI -  \tilde{\rmD}^{-1/2}\tilde{\rmA}\tilde{\rmD}^{-1/2}$. We have that
\begin{equation}
0 = \lambda_1 = \tilde{\lambda}_1 < \tilde{\lambda}_n < \lambda_n. \label{eq:corollary_aug_laplacian}
\end{equation}
\end{theorem}
\autoref{thm:Lapla_eig_shrink} shows that the largest eigenvalue of the normalized graph Laplacian becomes smaller after adding self-loops $\gamma > 0$ (see supplementary materials for the proof).

\autoref{fig:s_filters} depicts the filtering operations associated with the normalized adjacency $\rmS_{\text{adj}} = \rmD^{-1/2}\rmA\rmD^{-1/2}$ and its augmented variant $\tilde{\rmS}_{\text{adj}} =  \tilde{\rmD}^{-1/2}\tilde{\rmA}\tilde{\rmD}^{-1/2}$ on the Cora dataset~\citep{sen2008collective}. Feature propagation with $\rmS_{\text{adj}}$ corresponds to filters $g(\lambda_i) = (1-\lambda_i)^K$ in the spectral range $[0, 2]$; therefore odd powers of $\rmS_{\text{adj}}$ yield negative filter coefficients at frequencies $\lambda_i > 1$. By adding self-loops ($\tilde{\rmS}_{\text{adj}}$), the largest eigenvalue shrinks from $2$ to approximately $1.5$ and then eliminates the effect of negative coefficients. Moreover, this scaled spectrum allows the filter defined by taking powers $K>1$ of $\tilde{\rmS}_{\text{adj}}$ to act as a low-pass-type filters. In supplementary material, we empirically evaluate different choices for the propagation matrix.

\section{Related Works}
% In this section, we review and analyze related works. We begin with a review of graph convolutional models and then transition to graph attention models. Lastly, we discuss other works on graphs. 

\begin{figure*}[tb!] 
\centering
\includegraphics[width=0.8\textwidth]{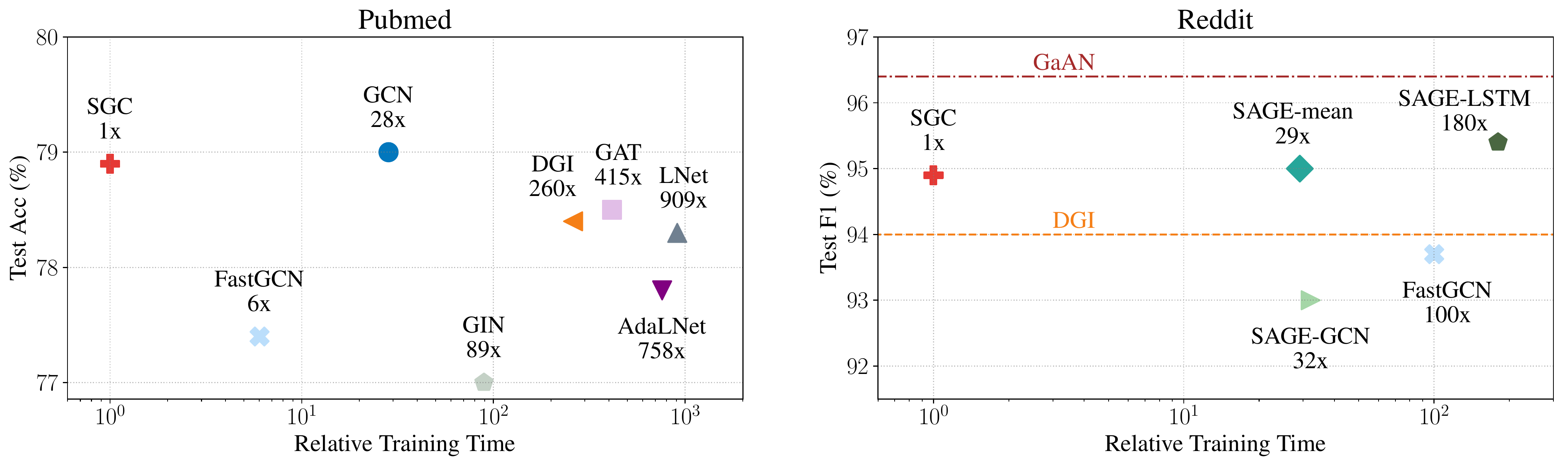}
\caption{Performance over training time on Pubmed and Reddit. \method{} is the fastest while achieving competitive performance. 
We are not able to benchmark the training time of GaAN and DGI on Reddit because the implementations are not released. 
}
\label{fig:run_time}
\end{figure*}
\subsection{Graph Neural Networks}
\citet{Bruna13} first propose a spectral graph-based extension of convolutional networks to graphs. 
In a follow-up work, ChebyNets \cite{Defferrard16} define graph convolutions using Chebyshev polynomials to remove the computationally expensive Laplacian eigendecomposition. GCNs \cite{gcn} further simplify graph convolutions by stacking layers of first-order Chebyshev polynomial filters with a redefined propagation matrix $\mathbf{S}$. 
\citet{FastGCN} propose an efficient variant of GCN based on importance sampling, and \citet{Hamilton17} propose a framework based on sampling and aggregation. 
\citet{dcnn}, \citet{n-gcn}, and \citet{liao2018lanczosnet} exploit multi-scale information by raising $\mathbf{S}$ to higher order.
% \citet{xu2018how} introduce Graph Isomorphism Networks which is claimed to be the most expressive variant of GNNs.
\citet{xu2018how} study the expressiveness of graph neural networks in terms of their ability to distinguish any two graphs and introduce Graph Isomorphism Network, which is proved to be as powerful as the Weisfeiler-Lehman test for graph isomorphism. 
\citet{Klicpera19} separate the non-linear transformation from propagation by using a neural network followed by a personalized random walk.
There are many other graph neural models~\cite{Monet, EP17, Li18}; we refer to \citet{gnn_review, battaglia2018relational, wu2019comprehensive} for a more comprehensive review. 

% GLN
% \citet{agnn} discover that a linear version of GCN can perform competitively and develop a attention-based GCN model.
% \citet{cai2018simple} propose an effective linear baseline for graph classification using node degree statistics.
% \citet{Buchnik18} show that self-training can improve the base linear graph model.
% \citet{Li2019LabelES} propose a generalized version of label propagation and provide a similar spectral analysis on the renormalization trick.

Previous publications have pointed out that simpler, sometimes linear models can be effective for node/graph classification tasks. \citet{agnn} empirically show that a linear version of GCN can perform competitively and propose an attention-based GCN variant. \citet{cai2018simple} propose an effective linear baseline for graph classification using node degree statistics. \citet{Buchnik18} show that models which use linear feature/label propagation steps can benefit from self-training strategies. 
\citet{Li2019LabelES} propose a generalized version of label propagation and provide a similar spectral analysis of the renormalization trick.
% Unlike these previous works, we ...

% Recently, \citet{agnn} proposed an attention-based GCN model for citation networks. The authors point out that a linear version of GCN could perform similarly to GCNs on these classification tasks, which matches our findings. \citet{cai2018simple} propose an effective linear baseline for non-attribute graph classification using simple node degree statistics.

% \subsection{Graph Attention Models}
Graph Attentional Models learn to assign different edge weights at each layer based on node features and have achieved state-of-the-art results on several graph learning tasks \citep{gat, agnn, zhang2018gaan, ADGPM}.
However, the attention mechanism usually adds significant overhead to computation and memory usage. 
We refer the readers to \citet{attention-survey} for further comparison.

\subsection{Other Works on Graphs} 
Graph methodologies can roughly be categorized into two approaches: graph embedding methods and graph laplacian regularization methods. 
Graph embedding methods \citep{Weston2008, Perozzi14, Yang16, infomax} represent nodes as high-dimensional feature vectors. 
Among them, DeepWalk~\citep{Perozzi14} and Deep Graph Infomax (DGI)~\citep{infomax} use unsupervised strategies to learn graph embeddings.
DeepWalk relies on truncated random walk and uses a skip-gram model to generate embeddings, whereas DGI trains a graph convolutional encoder through maximizing mutual information. 
Graph Laplacian regularization \citep{Zhu03, Zhou04,Belkin04b,Belkin2006} introduce a regularization term based on graph structure which forces nodes to have similar labels to their neighbors.
Label Propagation~\citep{Zhu03} makes predictions by spreading label information from labeled nodes to their neighbors until convergence. 

\section{Experiments and Discussion}
\label{sec:experiment}

We first evaluate \method{} on citation networks and social networks and then extend our empirical analysis to a wide range of downstream tasks.
\subsection{Citation Networks \& Social Networks} \label{sec:citation-networks}
We evaluate the semi-supervised node classification performance of \method{} on the Cora, Citeseer, and Pubmed citation network datasets (\autoref{table:citation-base}) \cite{sen2008collective}. 
We supplement our citation network analysis by using \method{} to inductively predict community structure on Reddit (\autoref{table:reddit}), which consists of a much larger graph. Dataset statistics are summarized in \autoref{table:citation-dataset}.

\paragraph{Datasets and experimental setup.}
\begin{table}[tb!]
\small
\centering
\caption{Dataset statistics of the citation networks and Reddit.}
\label{table:citation-dataset}
\begin{tabular}{l|cccccc}
\toprule
Dataset & \# Nodes & \# Edges & Train/Dev/Test Nodes \\
\midrule
Cora & $2,708$ & $5,429$ & $140/500/1,000$\\
Citeseer & $3,327$ & $4,732$ & $120/500/1,000$\\
Pubmed & $19,717$ & $44,338$ & $60/500/1,000$\\
\midrule
Reddit & $233$K & $11.6$M & $152$K/$24$K/$55$K\\
\bottomrule
\end{tabular}
\end{table}
\begin{table}[htb!]
\centering
        \small
        \caption{Test accuracy (\%) averaged over 10 runs on citation networks. $^\dagger$We remove the outliers (accuracy $< 75/65/75 \%$) when calculating their statistics due to high variance.}
        \label{table:citation-base}
        \begin{tabular}{l|c|c|c}
        \toprule
         & Cora & Citeseer & Pubmed \\ 
        \midrule
        \multicolumn{4}{l}{\textbf{Numbers from literature:}} \\
        % \midrule
        GCN  & $81.5$ & $70.3$ & $79.0$  \\
        GAT  & $83.0 \pm 0.7$ & $72.5 \pm{0.7}$ & $79.0 \pm{0.3}$ \\
        GLN  & $81.2 \pm 0.1$ & $70.9 \pm{0.1}$ & $78.9 \pm{0.1}$ \\
        AGNN & $83.1 \pm 0.1$ & $71.7 \pm{0.1}$ & $79.9 \pm{0.1}$ \\
        LNet & $79.5 \pm 1.8$ & $66.2 \pm 1.9$  & $78.3 \pm 0.3$ \\
        AdaLNet & $80.4 \pm 1.1$ & $68.7 \pm 1.0$  & $78.1 \pm 0.4$ \\
        DeepWalk & $70.7 \pm 0.6$ & $51.4 \pm 0.5$ & $76.8 \pm 0.6$\\
        DGI & $82.3 \pm 0.6$ & $71.8 \pm 0.7$ & $76.8 \pm 0.6$ \\
         \midrule
        \multicolumn{4}{l}{\textbf{Our experiments:}} \\
        % \midrule
        GCN & $81.4 \pm{0.4}$ & $70.9\pm 0.5$ & $79.0 \pm{0.4}$ \\
        % GCN - test w/o relu & $81.8 \pm{0.66}$ & $70.7\pm 0.93$ & $79.1 \pm{0.65}$ \\
        GAT & $83.3 \pm{0.7}$ & $72.6 \pm{0.6}$ &  $78.5 \pm{0.3}$ \\
        FastGCN & $79.8 \pm{0.3}$ & $68.8 \pm{0.6} $ & $77.4 \pm{0.3}$ \\
        GIN & $77.6 \pm 1.1$ &  $66.1 \pm 0.9$ & $77.0 \pm 1.2$ \\
        LNet & $80.2 \pm 3.0^\dagger$ & $67.3 \pm 0.5$  & $78.3 \pm 0.6^\dagger$ \\
        AdaLNet & $81.9 \pm 1.9^\dagger$ & $70.6 \pm 0.8^\dagger$ & $77.8 \pm 0.7^\dagger$ \\
        DGI & $82.5 \pm 0.7$ & $71.6 \pm 0.7 $ & $78.4 \pm 0.7$\\
        {\color{modelblue} \method{}} & $81.0 \pm 0.0$ & $71.9 \pm 0.1$ & $78.9 \pm 0.0$ \\
         \bottomrule
        \end{tabular}
\end{table}
\begin{table}[htb!]
        \centering
        \small
        \caption{Test Micro F1 Score (\%) averaged over 10 runs on Reddit. Performances of models are cited from their original papers. \textbf{OOM:} Out of memory.}
        \label{table:reddit}
        \begin{tabular}{l|l|l}
        \toprule
        Setting & Model & Test F1 \\
         \midrule
        \multirow{5}{*}{\shortstack[c]{Supervised}}
        & GaAN  & $96.4$ \\
        & SAGE-mean & $95.0$\\
        & SAGE-LSTM & $95.4$\\
        & SAGE-GCN & $93.0$\\
        & FastGCN & $93.7$\\
        & GCN & \textbf{OOM} \\
        \midrule
        \multirow{3}{*}{\shortstack[c]{Unsupervised}} 
        & SAGE-mean & $89.7$ \\
        & SAGE-LSTM & $90.7$\\
        & SAGE-GCN  & $90.8$\\
        & DGI & $94.0$\\
        \midrule
        \multirow{2}{*}{\shortstack[c]{No Learning}} 
        & Random-Init DGI & $93.3$ \\
        & {\color{modelblue} \method{}} & $94.9$ \\
        \bottomrule
        \end{tabular}
\end{table}
On the citation networks, we train \method{} for 100 epochs using Adam~\citep{adam} with learning rate 0.2. In addition, we use weight decay and tune this hyperparameter on each dataset using hyperopt~\citep{hyperopt} for 60 iterations on the public split validation set. 
Experiments on citation networks are conducted \emph{transductively}. 
On the Reddit dataset, we train \method{} with L-BFGS \cite{lbfgs} using no regularization, and remarkably, training converges in 2 steps. 
We evaluate \method{} \emph{inductively} by following \citet{FastGCN}: we train \method{} on a subgraph comprising only training nodes and test with the original graph.
On all datasets, we tune the number of epochs based on both convergence behavior and validation accuracy.
% We implemented \method{} with PyTorch~\cite{pytorch}.

\paragraph{Baselines.} For citation networks, we compare against GCN~\citep{gcn}
% \footnote{\url{https://github.com/matenure/FastGCN}}, 
GAT~\citep{gat}
% \footnote{\url{github.com/PetarV-/GAT}}, 
FastGCN~\citep{FastGCN}
% \footnote{\url{github.com/matenure/FastGCN}},
LNet, AdaLNet~\citep{liao2018lanczosnet} 
% \footnote{\url{github.com/lrjconan/LanczosNetwork}},
and DGI~\citep{infomax} using the publicly released implementations.
Since GIN is not initially evaluated on citation networks, we implement GIN following ~\citet{xu2018how} and use hyperopt to tune weight decay and learning rate for 60 iterations. 
Moreover, we tune the hidden dimension by hand.

For Reddit, we compare \method{} to the reported performance of GaAN~\cite{zhang2018gaan}, supervised and unsupervised variants of GraphSAGE~\cite{Hamilton17}, FastGCN, and DGI. \autoref{table:reddit} also highlights the setting of the feature extraction step for each method. 
We note that \method{} involves no learning because the feature extraction step, $\rmS^K\rmX$, has no parameter. Both unsupervised and no-learning approaches train logistic regression models with labels afterward.

\paragraph{Performance.}
Based on results in \autoref{table:citation-base} and \autoref{table:reddit}, we conclude that \method{} is very competitive. 
Table~\ref{table:citation-base} shows the performance of \method{} can match the performance of GCN and state-of-the-art graph networks on citation networks.
In particular on Citeseer, \method{} is about 1\% better than GCN, and we reason this performance boost is caused by \method{} having fewer parameters and therefore suffering less from overfitting.
Remarkably, GIN performs slight worse because of overfitting. Also, both LNet and AdaLNet are unstable on citation networks.
On Reddit, \autoref{table:reddit} shows that \method{} outperforms the previous sampling-based GCN variants, SAGE-GCN and FastGCN by more than 1\%. 

Notably, \citet{infomax} report that the performance of a randomly initialized DGI encoder nearly matches that of a trained encoder; however, both models underperform \method{} on Reddit.   
This result may suggest that the extra weights and nonlinearities in the DGI encoder are superfluous, if not outright detrimental. 

\paragraph{Efficiency.} 
In \autoref{fig:run_time}, we plot the performance of the state-of-the-arts graph networks over their training time relative to that of \method{} on the Pubmed and Reddit datasets. In particular, we precompute $\rmS^K\rmX$ and the training time of \method{} takes into account this precomputation time.
We measure the training time on a NVIDIA GTX 1080 Ti GPU and present the benchmark details in supplementary materials.

On large graphs (e.g. Reddit), GCN cannot be trained due to excessive memory requirements. 
Previous approaches tackle this limitation by either sampling to reduce neighborhood size~\cite{FastGCN, Hamilton17} or limiting their model sizes~\cite{infomax}.
By applying a fixed filter and precomputing $\rmS^K\rmX$, \method{} minimizes memory usage and only learns a single weight matrix during training. Since $\rmS$ is typically sparse and $K$ is usually small, we can exploit fast sparse-dense matrix multiplication to compute $\rmS^K\rmX$.
% $\rmS ( \ldots \rmS(\rmS \rmX)\ldots)$.
 \autoref{fig:run_time} shows that \method{} can be trained up to two orders of magnitude faster than fast sampling-based methods while having little or no drop in performance.

\subsection{Downstream Tasks}
We extend our empirical evaluation to 5 downstream applications --- text classification, semi-supervised user geolocation, relation extraction, zero-shot image classification, and graph classification --- to study the applicability of \method{}. 
We describe experimental setup in supplementary materials.

\begin{table}[tb!]
\centering
\small
\caption{Test Accuracy (\%) on text classification datasets. The numbers are averaged over 10 runs.}
\begin{tabular}{l|l|cc}
\toprule
Dataset & Model & Test Acc. $\uparrow$ & Time (seconds) $\downarrow$ \\
\midrule
\multirow{2}{*}{20NG} & GCN & $87.9 \pm{0.2}$ & $1205.1 \pm{144.5}$ \\ & {\color{modelblue} \method{}} & $88.5 \pm{0.1}$ & $19.06 \pm{0.15}$ \\
\midrule
\multirow{2}{*}{R8} & GCN & $97.0 \pm{0.2}$ & $129.6 \pm{9.9}$ \\ & {\color{modelblue} \method{}} & $97.2 \pm{0.1}$ & $1.90 \pm{0.03}$ \\
\midrule
\multirow{2}{*}{R52} & GCN & $93.8 \pm{0.2}$ & $245.0 \pm{13.0}$ \\ & {\color{modelblue} \method{}} & $94.0 \pm{0.2}$ & $3.01 \pm{0.01}$ \\
\midrule
\multirow{2}{*}{Ohsumed} & GCN & $68.2 \pm{0.4}$ & $252.4 \pm{14.7}$ \\ & {\color{modelblue} \method{}} & $68.5 \pm{0.3}$ & $3.02 \pm{0.02}$ \\
\midrule
\multirow{2}{*}{MR} & GCN & $76.3 \pm{0.3}$ & $16.1 \pm{0.4}$ \\ & {\color{modelblue} \method{}} & $75.9 \pm{0.3}$ & $4.00 \pm{0.04}$ \\
\bottomrule
\end{tabular}
\label{table:text-base-time}
\end{table}

\paragraph{Text classification} assigns labels to documents. 
\citet{textGCN} use a 2-layer GCN to achieve state-of-the-art results by creating a corpus-level graph which treats both documents and words as nodes in a graph. 
Word-word edge weights are pointwise mutual information (PMI) and word-document edge weights are normalized TF-IDF scores. 
\autoref{table:text-base-time} shows that an SGC ($K=2$) rivals their model on 5 benchmark datasets, while being up to $83.6\times$ faster.
\begin{table}[t!]
\centering
\small
\caption{Test accuracy (\%) within 161 miles on semi-supervised user geolocation. The numbers are averaged over 5 runs.}
\resizebox{\linewidth}{!}{%
\begin{tabular}{l|l|rrrr}
\toprule
Dataset & Model & Acc.@161$\uparrow$ & Time $\downarrow$ \\
\midrule
\multirow{2}{*}{GEOTEXT} & GCN+H & $60.6 \pm 0.2$ & $153.0$s\\
& {\color{modelblue} \method{}} & $61.1 \pm 0.1$ & $5.6$s\\
\midrule
\multirow{2}{*}{TWITTER-US} & GCN+H & $61.9 \pm 0.2$ & $9$h $54$m \\
& {\color{modelblue} \method{}} & $62.5 \pm 0.1$ & $4$h $33$m  \\
\midrule
\multirow{2}{*}{TWITTER-WORLD} & GCN+H & $53.6 \pm 0.2$ & $2$d $05$h $17$m \\
& {\color{modelblue} \method{}} & $54.1 \pm 0.2$ & $22$h $53$m \\
\bottomrule
\end{tabular}
}
\label{table:geo_result}
\end{table}

\paragraph{Semi-supervised user geolocation} locates the ``home'' position of users on social media given users' posts, connections among users, and a small number of labelled users. \citet{Rahimi18} apply GCNs with highway connections on this task and achieve close to state-of-the-art results.
\autoref{table:geo_result} shows that \method{} outperforms GCN with highway connections on GEOTEXT \citep{eisenstein2010latent}, TWITTER-US \citep{roller2012supervised}, and TWITTER-WORLD \citep{han2012geolocation} under \citet{Rahimi18}'s framework, while saving $30+$ hours on TWITTER-WORLD.

\begin{table}[t!]
\centering
\caption{Test Accuracy (\%) on Relation Extraction. The numbers are averaged over 10 runs.}
\label{table:relation-base}
\begin{tabular}{l|c}
\toprule
TACRED & Test Accuracy $\uparrow$ \\
\midrule
C-GCN \citep{relation-extraction} & $66.4$ \\
C-GCN & $66.4 \pm{0.4}$\\
{\color{modelblue} C-\method{}} & $67.0 \pm{0.4}$\\
 \bottomrule
\end{tabular}
\end{table}

\paragraph{Relation extraction} involves predicting the relation between subject and object in a sentence.
\citet{relation-extraction} propose C-GCN which uses an LSTM~\citep{LSTM} followed by a GCN and an MLP.
We replace GCN with SGC ($K=2$) and call the resulting model C-SGC. \autoref{table:relation-base} shows that C-SGC sets new state-of-the-art on TACRED~\citep{TACRED}.

\begin{table}[t!]
\centering
\small
\caption{Top-1 accuracy (\%) averaged over 10 runs in the 2-hop and 3-hop setting of the zero-shot image task on ImageNet. ADGPM~\citep{ADGPM} and EXEM 1-nns~\citep{EXEM} use more powerful visual features.}
\label{table:zero_shot}
\resizebox{\linewidth}{!}{%
\begin{tabular}{l|l|l|l}
\toprule
Model & \# Param. $\downarrow$ & 2-hop Acc. $\uparrow$ & 3-hop Acc. $\uparrow$\\
 \midrule
\multicolumn{4}{l}{\textbf{Unseen categories only:}} \\
EXEM 1-nns & - & $27.0$ & $7.1$ \\ 
ADGPM & - & $26.6$ & $6.3$ \\
GCNZ  & - & $19.8$  & $4.1$ \\
GCNZ (ours) & $9.5M$ & $20.9 \pm 0.2$  & $4.3 \pm 0.0$ \\
{\color{modelblue} MLP-\method{}Z (ours)}  & $4.3M$ &  $21.2 \pm 0.2$ & $4.4 \pm 0.1$ \\
\midrule
\multicolumn{4}{l}{\textbf{Unseen categories \& seen categories:}} \\
ADGPM & - & $10.3$ & $2.9$ \\
GCNZ  & - & $9.7$ & $2.2$ \\
GCNZ (ours) & $9.5M$ & $10.0 \pm 0.2$ & $2.4 \pm 0.0$ \\
{\color{modelblue} MLP-\method{}Z (ours)} & $4.3M$ &  $10.5 \pm 0.1$ & $2.5 \pm 0.0$ \\
 \bottomrule
\end{tabular}
}
\end{table}

\paragraph{Zero-shot image classification} consists of learning an image classifier without access to any images or labels from the test categories. 
GCNZ~\citep{wang2018zero} uses a GCN to map the category names --- based on their relations in WordNet~\citep{miller1995wordnet} --- to image feature domain, and find the most similar category to a query image feature vector.
\autoref{table:zero_shot} shows that replacing GCN with an MLP followed by \method{} can improve performance while reducing the number of parameters by $55\%$.
We find that an MLP feature extractor is necessary in order to map the pretrained GloVe vectors to the space of visual features extracted by a ResNet-50.
Again, this downstream application demonstrates that learned graph convolution filters are superfluous; similar to \citet{EXEM}'s observation that GCNs may not be necessary.

\paragraph{Graph classification} requires models to use graph structure to categorize graphs.
\citet{xu2018how} theoretically show that GCNs are not sufficient to distinguish certain graph structures and show that their GIN is more expressive and achieves state-of-the-art results on various graph classification datasets. We replace the GCN in DCGCN~\citep{zhang2018end} with an \method{} and get $71.0\%$ and $76.2\%$ on NCI1 and COLLAB datasets~\citep{yanardag2015deep} respectively, which is on par with an GCN counterpart, but far behind GIN. Similarly, on QM8 quantum chemistry dataset~\citep{ramakrishnan2015electronic}, more advanced AdaLNet and LNet~\citep{liao2018lanczosnet} get $0.01$ MAE on QM8, outperforming \method{}'s $0.03$ MAE by a large margin.

\section{Conclusion}
In order to better understand and explain the mechanisms of GCNs, we explore the simplest possible formulation of a graph convolutional model, \method{}. The algorithm is almost trivial, a graph based pre-processing step  followed by standard multi-class logistic regression. However, the performance of \method{} rivals --- if not surpasses --- the performance of GCNs and state-of-the-art graph neural network models across a wide range of graph learning tasks.
Moreover by precomputing the fixed feature extractor $\rmS^K$, training time is reduced to a record low.
For example on the Reddit dataset, \method{} can be trained up to two orders of magnitude faster than sampling-based GCN variants. 

In addition to our empirical analysis, we analyze \method{} from a convolution perspective and manifest this method as a low-pass-type filter on the spectral domain. Low-pass-type filters capture low-frequency signals, which corresponds with smoothing features across a graph in this setting. 
Our analysis also provides insight into the empirical boost of the ``renormalization trick" and demonstrates how shrinking the spectral domain leads to a low-pass-type filter which underpins \method{}. 

Ultimately, the strong performance of \method{} sheds light onto GCNs. It is likely that the expressive power of GCNs originates primarily from the repeated graph propagation (which \method{} preserves) rather than the nonlinear feature extraction (which it doesn't.) 

Given its empirical performance, efficiency, and interpretability, we argue that the \method{} should be highly beneficial to the community in at least three ways:
(1) as a first model to try, especially for node classification tasks; 
(2) as a simple baseline for comparison with future graph learning models; 
(3) as a starting point for future research in graph learning --- returning to the historic machine learning practice to develop complex from simple models.

\section*{Acknowledgement}
This research is supported in part by grants from the National
Science Foundation (III-1618134, III-1526012, IIS1149882,
IIS-1724282, and TRIPODS-1740822), the Office
of Naval Research DOD (N00014-17-1-2175), 
Bill and Melinda Gates Foundation, and 
Facebook Research. We are thankful for
generous support by SAP America Inc. 
Amauri Holanda de Souza Jr. thanks CNPq (Brazilian Council for Scientific and Technological Development) for the financial support.
We appreciate the discussion with Xiang Fu, Shengyuan Hu, Shangdi Yu, Wei-Lun Chao and Geoff Pleiss as well as the figure design support from Boyi Li.
% Last but not least, we would like to thank Johannes Klicpera, Aleksandar Bojchevski, Stephan Gunnemann, and Xiao-Ming Wu for helping us identify the similarities and differences between this work and their papers~\citep{Klicpera19,Li2019LabelES}.

% \newpage

\bibliography{references}
\bibliographystyle{icml2019}

\clearpage

\twocolumn[   %or twocolumn

\icmltitle{Simplifying Graph Convolutional Networks \\ (Supplementary Material)}

]

\appendix

\section{The spectrum of $\tilde{\boldsymbol{\Delta}}_{\text{sym}}$}

The normalized Laplacian defined on graphs with self-loops, $\tilde{\boldsymbol{\Delta}}_{\text{sym}}$, consists of an instance of generalized graph Laplacians and hold the interpretation as a difference operator, i.e. for any signal $\rvx \in \mathbb{R}^n$ it satisfies 
\begin{equation}
(\tilde{\boldsymbol{\Delta}}_{\text{sym}} \rvx)_i = \sum_{j} \frac{\tilde{a}_{ij}}{\sqrt{d_i + \gamma}} \left(\frac{x_i}{\sqrt{d_i + \gamma}} - \frac{x_j}{\sqrt{d_j + \gamma}}\right). \nonumber
\end{equation}

Here, we prove several properties regarding its spectrum.

\begin{lemma}
(Non-negativity of $\tilde{\boldsymbol{\Delta}}_{\text{sym}}$) The augmented normalized Laplacian matrix is symmetric positive semi-definite.
\end{lemma}
\begin{proof}

The quadratic form associated with $\tilde{\boldsymbol{\Delta}}_{\text{sym}}$ is
\begin{align}
    & \rvx^\top\tilde{\boldsymbol{\Delta}}_{\text{sym}}\rvx  = \sum_i x_i^2 - \sum_i \sum_j\frac{\tilde{a}_{ij} x_i x_j}{\sqrt{(d_i + \gamma)(d_j + \gamma)}} \nonumber \\
    & = \frac{1}{2} \left( \sum_i x_i^2 + \sum_j x_j^2  - \sum_i \sum_j \frac{2\tilde{a}_{ij} x_i x_j}{\sqrt{(d_i + \gamma)(d_j + \gamma)}}\right) \nonumber \\
    & = \frac{1}{2} \left(\sum_i \sum_j \frac{\tilde{a}_{ij} x_i^2}{d_i + \gamma} + \sum_j \sum_i \frac{\tilde{a}_{ij} x_j^2}{d_j + \gamma} \right.
     \nonumber \\ & \left. \quad \quad - \sum_i \sum_j  \frac{2\tilde{a}_{ij} x_i x_j}{\sqrt{(d_i + \gamma)(d_j + \gamma)}}\right) \nonumber \\
    &  = \frac{1}{2} \sum_i \sum_j \tilde{a}_{ij}\left(\frac{x_i}{\sqrt{d_i + \gamma}} - \frac{x_j}{\sqrt{d_j + \gamma}}\right)^2 \geq 0 \label{eq:norm_laplacian_self_loops}
\end{align}

\end{proof}

\begin{lemma}
\label{lem:0_eig}
$0$ is an eigenvalue of both $\boldsymbol{\Delta}_{\text{sym}}$ and $\tilde{\boldsymbol{\Delta}}_{\text{sym}}$.
\end{lemma}
\begin{proof}
First, note that $\rvv=[1,\ldots,1]^\top$ is an eigenvector of $\boldsymbol{\Delta}$ associated with eigenvalue $0$, i.e., $\boldsymbol{\Delta} \rvv = (\rmD-\rmA)\rvv = \mathbf{0}$.

Also, we have that $\tilde{\boldsymbol{\Delta}}_{\text{sym}} = \tilde{\rmD}^{-1/2}(\tilde{\rmD} - \tilde{\rmA})\tilde{\rmD}^{-1/2} = \tilde{\rmD}^{-1/2}\boldsymbol{\Delta}\tilde{\rmD}^{-1/2}$. Denote $\rvv_1=\tilde{\rmD}^{1/2}\rvv$, then
$$
\tilde{\boldsymbol{\Delta}}_{\text{sym}}\rvv_1 = \tilde{\rmD}^{-1/2}\boldsymbol{\Delta}\tilde{\rmD}^{-1/2}\tilde{\rmD}^{1/2}\rvv = \tilde{\rmD}^{-1/2} \boldsymbol{\Delta} \rvv = \mathbf{0}.
$$

Therefore, $\rvv_1=\tilde{\rmD}^{1/2}\rvv$ is an eigenvector of $\tilde{\boldsymbol{\Delta}}_{\text{sym}}$ associated with eigenvalue $0$, which is then the smallest eigenvalue from the non-negativity of $\tilde{\boldsymbol{\Delta}}_{\text{sym}}$. Likewise, 0 can be proved to be the smallest eigenvalues of $\boldsymbol{\Delta}_{\text{sym}}$.
\end{proof}

\begin{lemma}
\label{lem:adj_eig}
Let $\beta_1 \leq \beta_2 \leq \dots \leq \beta_n$ denote eigenvalues of $\rmD^{-1/2}\rmA \rmD^{-1/2}$ and $\alpha_1 \leq \alpha_2 \leq \dots \leq \alpha_n$ be the eigenvalues of $\tilde{\rmD}^{-1/2}\rmA\tilde{\rmD}^{-1/2}$. Then,
\begin{align} \label{eq:bounds_norm_adj}
    & \alpha_1 \geq \frac{\max_i d_i}{\gamma+\max_i d_i}\beta_1, &\alpha_n \leq \frac{\min_i{d_i}}{\gamma + \min_i{d_i}}.
\end{align}
\end{lemma}

\begin{proof}

We have shown that 0 is an eigenvalue of $\boldsymbol{\Delta}_{\text{sym}}$. Since $\rmD^{-1/2}\rmA \rmD^{-1/2} = \rmI - \boldsymbol{\Delta}_{\text{sym}}$, then $1$ is an eigenvalue of $\rmD^{-1/2}\rmA \rmD^{-1/2}$. More specifically, $\beta_n = 1$. In addition, by combining the fact that $\Tr(\rmD^{-1/2}\rmA\rmD^{-1/2})=0=\sum_i \beta_i$ with $\beta_n = 1$, we conclude that $\beta_1 < 0$.

By choosing $\rvx$ such that $\lVert \rvx \rVert =1$ and $\rvy=\rmD^{1/2}\tilde{\rmD}^{-1/2} \rvx$, we have that $\|\rvy\|^2=\sum\limits_i\frac{d_i}{d_i+\gamma}x_i^2$ and $\frac{\min_i d_i}{\gamma+\min_i d_i}\leq \|\rvy\|^2\leq \frac{\max_i d_i}{\gamma+\max_i d_i}$. Hence, we use the Rayleigh quotient to provide a lower bound to $\alpha_1$:
\begin{align*}
\alpha_1 & = \min_{\|\rvx\|=1} \left(\rvx^\top\tilde{\rmD}^{-1/2} \rmA \tilde{\rmD}^{-1/2} \rvx \right) \\
& = \min_{\|\rvx\|=1} \left( \rvy^\top\rmD^{-1/2} \rmA \rmD^{-1/2} \rvy \right) \text{(by replacing variable)} \\
&= \min_{\|\rvx\|=1} \left( \frac{\rvy^\top\rmD^{-1/2} \rmA \rmD^{-1/2}\rvy}{\|\rvy\|^2}\|\rvy\|^2 \right) \\
&\geq \min_{\|\rvx\|=1} \left( \frac{\rvy^\top\rmD^{-1/2} \rmA \rmD^{-1/2} \rvy}{\|\rvy\|^2} \right) \max_{\|\rvx\|=1} \left( \|\rvy\|^2 \right) \\
& ( \because  \min (AB) \geq \min (A) \max(B) \text{ if } \min (A) < 0, \forall B > 0,
\\
&\text{ and \quad} 
\min_{\|\rvx\|=1} \left( \frac{\rvy^\top\rmD^{-1/2} \rmA \rmD^{-1/2} \rvy}{\|\rvy\|^2} \right) = \beta_1 < 0 ) \\
&= \beta_1\max_{\|\rvx\|=1} \|\rvy\|^2 \\
&\geq \frac{\max_i d_i}{\gamma+\max_i d_i}\beta_1.\\
 \end{align*}

One may employ similar steps to prove the second inequality in \autoref{eq:bounds_norm_adj}.

\end{proof}

\begin{proof} [Proof of Theorem 1] 
Note that $\tilde{\boldsymbol{\Delta}}_{\text{sym}} = \rmI - \gamma \tilde{\rmD}^{-1} - \tilde{\rmD}^{-1/2}\rmA\tilde{\rmD}^{-1/2}$. Using the results in Lemma \autoref{lem:adj_eig}, we show that the largest eigenvalue $\tilde{\lambda}_n$ of $\tilde{\boldsymbol{\Delta}}_{\text{sym}}$ is
\begin{align} \label{eq:bound_laplacians}
        \tilde{\lambda}_n  & = \max_{\|\rvx\|=1} \rvx^\top(\rmI - \gamma \tilde{\rmD}^{-1} - \tilde{\rmD}^{-1/2}\rmA\tilde{\rmD}^{-1/2})\rvx  \nonumber \\
               & \leq  1 - \min_{\|\rvx\|=1} \gamma \rvx^\top \tilde{\rmD}^{-1} \rvx - \min_{\|\rvx\|=1} \rvx^\top \tilde{\rmD}^{-1/2} \rmA \tilde{\rmD}^{-1/2} \rvx \nonumber \\
               & = 1 - \frac{\gamma}{\gamma + \max_{i} d_i} - \alpha_1 \nonumber \\
               & \leq  1 - \frac{\gamma}{\gamma + \max_{i} d_i} - \frac{\max_i d_i}{\gamma + \max_i d_i} \beta_1 \nonumber \\
               & <  1 - \frac{\max_i d_i}{\gamma + \max_i d_i} \beta_1 \quad (\gamma > 0 \text{ and } \beta_1 < 0) \nonumber \\
               & < 1 - \beta_1 = \lambda_n
\end{align}

\end{proof}

\section{Experiment Details}
\label{sec:exp-details}
\paragraph{Node Classification.}
We empirically find that on Reddit dataset for \method{}, it is crucial to normalize the features into zero mean and univariate. 

\paragraph{Training Time Benchmarking.} We hereby describe the experiment setup of Figure 3.
\citet{FastGCN} benchmark the training time of FastGCN on CPU, and as a result, it is difficult to compare numerical values across reports.
Moreover, we found the performance of FastGCN improved with a smaller early stopping window (10 epochs); therefore, we could decrease the model's training time.
We provide the data underpinning Figure 3 in \autoref{table:citation-time} and \autoref{table:reddit-time}.
\begin{table}[htb!]
\centering
        \small
        \caption{Training time (seconds) of graph neural networks on Citation Networks. Numbers are averaged over 10 runs.}
        \label{table:citation-time}
        \begin{tabular}{l|c|c|c}
        \toprule
        Models & Cora & Citeseer & Pubmed \\ 
        \midrule
        % \midrule
        GCN & $0.49$ & $0.59$ & $8.31$ \\
        % GCN - test w/o relu & $81.8 \pm{0.66}$ & $70.7\pm 0.93$ & $79.1 \pm{0.65}$ \\
        GAT & $63.10$ & $118.10$ &  $121.74$ \\
        FastGCN & $2.47$ & $3.96 $ & $1.77$ \\
        GIN & $2.09$ &  $4.47$ & $26.15$ \\
        LNet & $15.02$ & $49.16$  & $266.47$ \\
        AdaLNet & $10.15$ & $31.80$ & $222.21$ \\
        DGI & $21.24$ & $21.06$ & $76.20$\\
        {\color{modelblue} \method{}} & $0.13$ & $0.14$ & $0.29$ \\
         \bottomrule
        \end{tabular}
\end{table}
\begin{table}[htb!]
        \centering
        \small
        \caption{Training time (seconds) on Reddit dataset.}
        \label{table:reddit-time}
        \begin{tabular}{l|l}
        \toprule
        Model & Time(s) $\downarrow$ \\
         \midrule
        SAGE-mean & $78.54$\\
        SAGE-LSTM & $486.53$\\
        SAGE-GCN & $86.86$\\
        FastGCN & $270.45$\\
        {\color{modelblue} \method{}} & $2.70$ \\
        \bottomrule
        \end{tabular}
\end{table}
\begin{figure*}[htb] 
\centering
\includegraphics[width=0.9\textwidth]{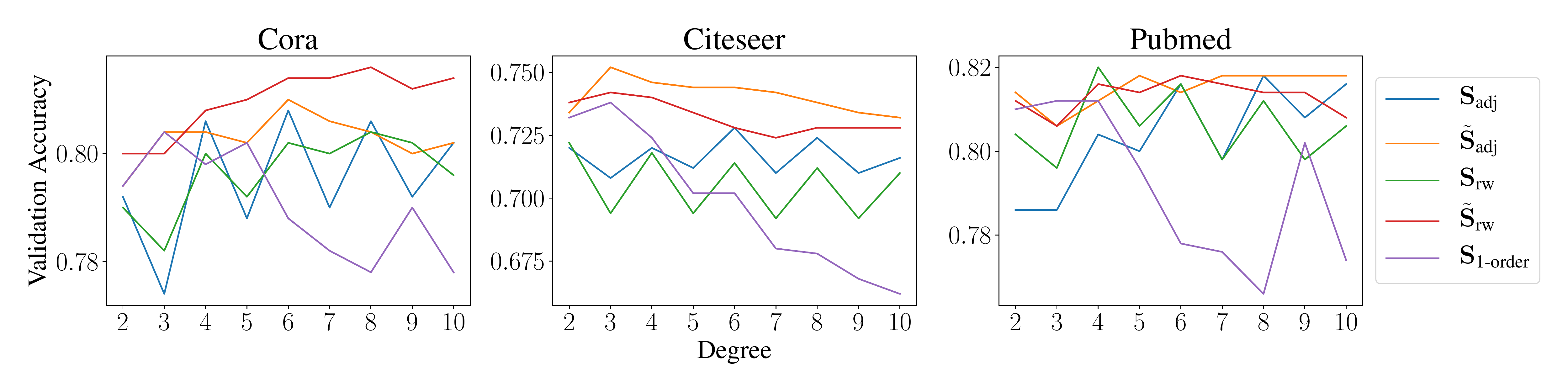}
\caption{Validation accuracy with \method{} using different propagation matrices.}
\label{fig:propagation-ablation}
\end{figure*}
\paragraph{Text Classification.} 
\citet{textGCN} use one-hot features for the word and document nodes. In training SGC, we normalize the features to be between 0 and 1 \textbf{after propagation} and train with L-BFGS for 3 steps. We tune the only hyperparameter, weight decay, using hyperopt\cite{hyperopt} for 60 iterations. Note that we cannot apply this feature normalization for TextGCN because the propagation cannot be precomputed. 
\paragraph{Semi-supervised User Geolocation.}
We replace the 4-layer, highway-connection GCN with a 3rd degree propagation matrix ($K=3$) SGC and use the same set of hyperparameters as \citet{Rahimi18}. All experiments on the GEOTEXT dataset are conducted on a single Nvidia GTX-1080Ti GPU while the ones on the TWITTER-NA and TWITTER-WORLD datasets are excuded with 10 cores of the Intel(R) Xeon(R) Silver 4114 CPU (2.20GHz). Instead of collapsing all linear transformations, we keep two of them which we find performing slightly better possibly due to . Despite of this subtle variation, the model is still linear.
\paragraph{Relation Extraction.}
We replace the 2-layer GCN with a 2nd degree propagation matrix ($K=2$) SGC and remove the intermediate dropout. We keep other hyperparameters unchanged, including learning rate and regularization. Similar to \citet{relation-extraction}, we report the best validation accuracy with early stopping.
\paragraph{Zero-shot Image Classification.}
We replace the 6-layer GCN (hidden size: 2048, 2048, 1024, 1024, 512, 2048) baseline with an 6-layer MLP (hidden size: 512, 512, 512, 1024, 1024, 2048) followed by a SGC with $K=6$. Following \cite{wang2018zero}, we only apply dropout to the output of SGC. Due to the slow evaluation of this task, we do not tune the dropout rate or other hyperparameters. Rather, we follow the GCNZ code and use learning rate of 0.001, weight decay of 0.0005, and dropout rate of 0.5. We also train the models with ADAM~\cite{adam} for 300 epochs.

\section{Additional Experiments}
\paragraph{Random Splits for Citation Networks.}
Possibly due to their limited size, the citation networks are known to be unstable. 
Accordingly, we conduct an additional 10 experiments on random splits of the training set while maintaining the same validation and test sets. 
\begin{table}[th!]
\small
\centering
\caption{Test accuracy (\%) on citation networks (random splits). $^\dagger$We remove the outliers (accuracy $< 0.7/0.65/0.75$) when calculating their statistics due to high variance.}
\label{table:citation-random}
\begin{tabular}{l|c|c|c}
\toprule
 & Cora & Citeseer & Pubmed \\ 
\midrule
\multicolumn{4}{l}{\textbf{Ours:}} \\
GCN &  $80.53 \pm{1.40}$ & $70.67\pm{2.90}$ & $77.09 \pm{2.95}$\\
% GCN - test w/o relu  & $80.18 \pm{1.65}$ & $70.98 \pm 2.93$  & $77.14 \pm{2.99}$\\
GIN  & $76.94 \pm 1.24$ & $66.56 \pm 2.27$ & $74.46 \pm 2.19$ \\
LNet & $74.23 \pm 4.50^\dagger$ & $67.26 \pm 0.81^\dagger$ & $77.20 \pm 2.03^\dagger$ \\
AdaLNet & $72.68 \pm 1.45^\dagger$ & $71.04 \pm 0.95^\dagger$ & $77.53 \pm 1.76^\dagger$ \\
GAT  & $82.29 \pm{1.16}$ & $72.6 \pm{0.58}$  & $78.79 \pm{1.41}$ \\
{\color{modelblue} \method{}} & $80.62 \pm{1.21}$ & $71.40 \pm{3.92}$ & $77.02 \pm{1.62} $\\
 \bottomrule
\end{tabular}
\end{table}
\paragraph{Propagation choice.}
We conduct an ablation study with different choices of propagation matrix, namely:
\begin{itemize}
\item[] Normalized Adjacency: $\mathbf{S}_{\text{adj}} = \rmD^{-1/2}\rmA \rmD^{-1/2}$
\item[] Random Walk Adjacency $\mathbf{S}_{\text{rw}} = \rmD^{-1}\rmA$
\item[] Aug. Normalized Adjacency $\tilde{\mathbf{S}}_{\text{adj}} = \tilde{\rmD}^{-1/2}\tilde{\rmA} \tilde{\rmD}^{-1/2}$ \item[] Aug. Random Walk $\tilde{\mathbf{S}}_{\text{rw}} = \tilde{\rmD}^{-1} \tilde{\rmA}$ 
\item[] First-Order Cheby $\mathbf{S}_{\text{1-order}}=(\rmI + \rmD^{-1/2}\rmA \rmD^{-1/2} )$
\end{itemize}

We investigate the effect of propagation steps $K \in \{2..10\}$ on validation set accuracy. 
We use hyperopt to tune L2-regularization and leave all other hyperparameters unchanged. \autoref{fig:propagation-ablation} depicts the validation results achieved by varying the degree of different propagation matrices.

We see that augmented propagation matrices (i.e. those with self-loops) attain higher accuracy and more stable performance across various propagation depths. Specifically, the accuracy of $\mathbf{S}_{\text{1-order}}$ tends to deteriorate as the power $K$ increases, and this results suggests using large filter coefficients on low frequencies degrades \method{} performance on semi-supervised tasks.

Another pattern is that odd powers of $K$ cause a significant performance drop for the normalized adjacency and random walk propagation matrices. This demonstrates how odd powers of the un-augmented propagation matrix use negative filter coefficients on high frequency information. Adding self-loops to the propagation matrix shrinks the spectrum such that the largest eigenvalues decrease from $\approx 2$ to $\approx 1.5$ on the citation network datasets. By effectively shrinking the spectrum, the effect of negative filter coefficients on high frequencies is minimized, and as a result, using odd-powers of $K$ does not degrade the performance of augmented propagation matrices. For non-augmented propagation matrices --- where the largest eigenvalue is approximately 2 --- negative coefficients significantly distort the signal, which leads to decreased accuracy. Therefore, adding self-loops constructs a better domain in which fixed filters can operate. 

\begin{table}[h]
    \centering
    \begin{tabular}{c|cc}
    \toprule
    \# Training Samples & \method{} & GCN \\
    \midrule
    1 & 33.16 & 32.94 \\
    5 & 63.74 & 60.68 \\
    10 & 72.04 & 71.46 \\
    20 & 80.30 & 80.16 \\
    40 & 85.56 & 85.38 \\
    80 & 90.08 & 90.44 \\
    \bottomrule
    \end{tabular}
    \caption{Validation Accuracy (\%) when \method{} and GCN are trained with different amounts of data on Cora. The validation accuracy is averaged over 10 random training splits such that each class has the same number of training examples.} 
    \label{tab:data_ablation}
\end{table}

\paragraph{Data amount.}
We also investigated the effect of training dataset size on accuracy. 
As demonstrated in Table~\ref{tab:data_ablation}, \method{} continues to perform similarly to GCN as the training dataset size is reduced, and even outperforms GCN when there are fewer than $5$ training samples. We reason this study demonstrates \method{} has at least the same modeling capacity as GCN.
%

%%%%%%%%%%%%%%%%%%%%%%%%%%%%%%%%%%%%%%%%%%%%%%%%%%%%%%%%%%%%%%%%%%%%%%%%%%%%%%%
%%%%%%%%%%%%%%%%%%%%%%%%%%%%%%%%%%%%%%%%%%%%%%%%%%%%%%%%%%%%%%%%%%%%%%%%%%%%%%%

\end{document}